\documentclass[letterpaper]{article} 
\usepackage{aaai2026}  
\usepackage{times}  
\usepackage{helvet}  
\usepackage{courier}  
\usepackage[hyphens]{url}  
\usepackage{graphicx} 
\usepackage{caption} 
\usepackage{algorithm}
\usepackage{algorithmic}
\usepackage{amsmath}
\usepackage{amssymb}
\usepackage{booktabs}
\usepackage{amsthm}          
\newtheorem{theorem}{Theorem}     

\usepackage{newfloat}
\usepackage{listings}
\DeclareCaptionStyle{ruled}{labelfont=normalfont,labelsep=colon,strut=off} 
\floatstyle{ruled}
\newfloat{listing}{tb}{lst}{}
\floatname{listing}{Listing}
\title{SinBasis Networks: Matrix‑Equivalent Feature Extraction for Wave‑Like Optical Spectrograms}
\author {
    Yuzhou Zhu\textsuperscript{\rm 1}\thanks{Corresponding author.},
    Zheng Zhang\textsuperscript{\rm 1},
    Ruyi Zhang\textsuperscript{\rm 2},
    Liang Zhou\textsuperscript{\rm 1 \rm *}
}
\affiliations {
    \textsuperscript{\rm 1}Dalian University of Technology\\
    \textsuperscript{\rm 2}Zhejiang Gongshang University\\
    1730694701@mail.dlut.edu.cn, zhouliang@dlut.edu.cn \\
}

\usepackage{bibentry}

\begin{document}

\maketitle

\begin{abstract}
Wave‑like images—from attosecond streaking spectrograms to optical spectra, audio mel‑spectrograms and periodic video frames—encode critical harmonic structures that elude conventional feature extractors. We propose a unified, matrix‑equivalent framework that reinterprets convolution and attention as linear transforms on flattened inputs, revealing filter weights as basis vectors spanning latent feature subspaces. To infuse spectral priors we apply elementwise \(\sin(\cdot)\) mappings to each weight matrix. Embedding these transforms into CNN, ViT and Capsule architectures yields Sin‑Basis Networks with heightened sensitivity to periodic motifs and built‑in invariance to spatial shifts. Experiments on a diverse collection of wave‑like image datasets—including 80,000 synthetic attosecond streaking spectrograms, thousands of Raman, photoluminescence and FTIR spectra, mel‑spectrograms from AudioSet and cycle‑pattern frames from Kinetics—demonstrate substantial gains in reconstruction accuracy, translational robustness and zero‑shot cross‑domain transfer. Theoretical analysis via matrix isomorphism and Mercer‑kernel truncation quantifies how sinusoidal reparametrization enriches expressivity while preserving stability in data‑scarce regimes. Sin‑Basis Networks thus offer a lightweight, physics‑informed approach to deep learning across all wave‑form imaging modalities.
\end{abstract}

\begin{links}
     \link{Code}{https://github.com/Yuzhou541/SinBasis}
\end{links}

\section{Introduction}

Spectral images in optical physics often reveal wave-like fringes that encode key dynamical information \cite{white2019,zhu2025}. Attosecond streaking phase retrieval exemplifies an inverse problem in which an XUV pulse impresses sinusoidal delays onto photoelectron spectra via synchronized IR fields \cite{white2019}. While convolutional networks excel at local contrast detection \cite{lecun1998}, they lack innate mechanisms for global periodicity. Vision transformers capture long-range context \cite{vaswani2017} but may overlook the harmonic signatures that define many spectrograms. Capsule networks preserve part–whole hierarchies \cite{sabour2017} yet require explicit priors to honor continuous wavefront shifts.

A matrix-isomorphism perspective unifies these operations as linear transforms on flattened inputs \cite{zhang2025}. Inspired by fixed sinusoidal embeddings in positional encoding and random Fourier features \cite{tancik2020fourier,rahimi2008}, we propose applying elementwise \(\sin(\cdot)\) mappings to learned weight matrices (Figure \ref{fig:SinBasis Networks}). Unlike learnable Fourier features that optimize frequency parameters \cite{tancik2020fourier} or Fourier neural operators that integrate spectral filters in PDE solvers \cite{li2020fourier}, our Sin-Basis transform imposes a lightweight, parameter-free nonlinearity. Compared to PINN-style frequency regularization \cite{raissi2019pinn}, a fixed sine reparametrization preserves norm bounds and simplifies training while retaining phase sensitivity.

\begin{figure}[t] 
  \centering
  \includegraphics[clip,trim=0cm 7cm 0cm 1.8cm,width=\linewidth]{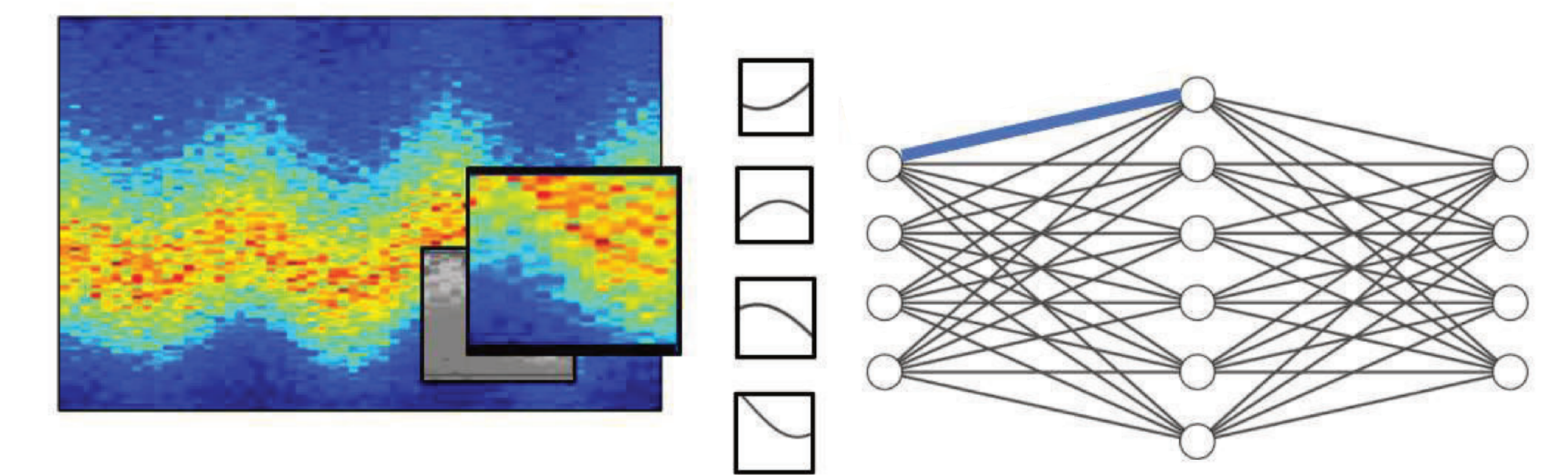}
  \caption{SinBasis Networks}
  \label{fig:SinBasis Networks}
\end{figure}

Embedding these transforms into CNN, ViT and Capsule architectures yields Sin-Basis Networks that amplify harmonic motifs and exhibit a \emph{stable shift response} under circular-shift/BTTB assumptions—\emph{not} strict shift invariance. Evaluations on 80,000 synthetic attosecond spectrograms and a suite of optical spectrum benchmarks demonstrate consistent gains in phase-retrieval accuracy, reduced mid/high-band error, robustness to attacks (PGD-20/AutoAttack/C\&W), and competitive cross-domain transfer. All headline results are reported as \textbf{mean$\pm$std over \(N{=}5\) seeds} with Wilcoxon signed-rank tests. Scope: on non-periodic natural images (e.g., ImageNet-100) we observe small drops (CNN/ViT: \(-0.6/-0.2\) pp), clarifying that our method targets wave-form modalities with pronounced harmonic structure. Theoretical analysis under the matrix-equivalent lens clarifies how sinusoidal bases enrich feature subspaces, with formal statements and assumptions detailed in the appendix.

\section{Related Work}

\subsection{Traditional optical spectrogram reconstruction (e.g., FROG-CRAB)}
Traditional optical spectrogram reconstruction draws upon frequency-resolved optical gating and its attosecond extension FROG-CRAB. Frequency-resolved optical gating measures both intensity and phase of ultrashort pulses via a spectrally resolved autocorrelation in a nonlinear medium followed by iterative phase retrieval \cite{trebino1993}. Subsequent implementations employed second-harmonic generation to overcome signal limitations of third-order processes \cite{delong1994}. The attosecond adaptation merges streaking spectrograms with generalized projections for complete retrieval of extreme ultraviolet bursts \cite{mairesse2005}\cite{quere2005}. Widely used algorithms include principal component generalized projections, Volkov-type transforms, and ptychographic engines to reconstruct temporal envelopes and spectral phase from measured traces. These classical pipelines provide strong domain baselines and motivate data-driven approaches when traces are noisy, under-sampled, or misspecified.

\subsection{Deep-learning approaches for spectral image processing (CNN, ViT, Capsule)}
Convolutional neural networks exploit localized receptive fields to detect fine-scale patterns in spectrograms \cite{white2019,zhu2025}. Their hierarchical filters successfully isolate fringe contrasts yet can struggle when global periodicities and subtle spatial shifts dominate the signal. Vision transformers leverage self-attention to model long-range dependencies across flattened image patches \cite{dosovitskiy2020,vaswani2017}; however, the absence of explicit periodic priors may lead to missed wavefront structure that is crucial for optical spectra. Capsule networks introduce vectorized capsules and dynamic routing to preserve part–whole relationships and pose information \cite{sabour2017}. This mechanism enhances spatial hierarchy modeling but still benefits from priors that favor sinusoidal motifs and continuous translations common to wave-like images.

\subsection{Matrix-equivalent interpretation of convolution and attention}
Viewing convolution as a linear transformation reveals that applying a local kernel across a flattened image vector equals left multiplication by a sparse banded matrix; each row shares filter weights and spatial offsets, reproducing sliding-window aggregation in a single multiplication \cite{zhang2025}. This algebraic lens clarifies how convolutional filters span a feature subspace. We adopt this view under standard assumptions (flattening via \texttt{im2col}, block Toeplitz-with-Toeplitz-blocks structure, and circular-shift reasoning), using it to reason about basis mappings—not to claim strict computational equivalence beyond these conditions. For self-attention, the projections \(Q{=}XW_q\), \(K{=}XW_k\), \(V{=}XW_v\) and the aggregation \(\mathrm{softmax}(QK^{\top}/\sqrt{d})V\) can likewise be written as dense matrix operations \cite{vaswani2017}; nevertheless, attention remains data-dependent due to the softmax, so the matrix-equivalent view serves as a unifying abstraction for analyzing basis effects.

\subsection{Fourier and sinusoidal bases for signal representation}
Decomposing signals into orthogonal sinusoids lies at the heart of Fourier analysis. Any wave-like pattern admits an expansion
\[
f(x)=\sum_{k}\!\big(a_{k}\sin(\omega_{k}x)+b_{k}\cos(\omega_{k}x)\big),
\]
which reveals inherent periodic components \cite{bracewell1986}. Discrete formulations extend these ideas to sampled data, yielding efficient algorithms and basis functions for image processing \cite{oppenheim1999}. Explicit sinusoidal embeddings further encode positional priors within deep architectures. Building on these insights, we embed an elementwise sine mapping into learned weight matrices to enrich the effective feature basis for wave-like spectrograms, targeting scenarios with pronounced harmonic content.

\subsection{Learnable Fourier Features, Fourier Neural Operators, and Frequency-Domain Regularization}
Recent works explore parameterized spectral mappings to capture high-frequency variations. Learnable Fourier Features introduce tunable frequency embeddings \(\gamma(x)=[\sin(2\pi Bx),\cos(2\pi Bx)]\) with learned \(B\), enabling representation of fine-scale detail in low-dimensional domains \cite{tancik2020fourier}. Fourier Neural Operators generalize this idea by performing global convolution in the Fourier domain: inputs undergo FFT, are filtered by learned spectral weights, and return via inverse FFT, supporting mesh-independent operator learning for PDEs \cite{li2020fourier}. In physics-informed neural networks, frequency-domain regularization penalizes spectral bias by encouraging outputs to match known spectra \cite{raissi2019pinn}.

By contrast, our Sin-Basis transform applies a fixed elementwise \(\sin(\cdot)\) to \emph{weight matrices}. This parameter-free weight-space mapping circumvents learning input-side bases or introducing extra spectral losses, while preserving simple training dynamics. Empirically, it provides competitive high-frequency expressivity and stable optimization for wave-form imagery, complementing CNN, ViT, and Capsule backbones without altering their parameter counts.

\section{Methodology}

Our Sin-Basis framework consists of three key components:
(1) a unified matrix-equivalent interpretation of standard layers,
(2) a fixed \(\sin(\cdot)\) reparameterization of weight matrices, and
(3) integration into CNN, ViT, and Capsule architectures.

\subsection{Unified Linear Representation}
We cast any \(m\times n\) spectrogram into a vector \(X\in\mathbb{R}^{L}\) with \(L=mn\) (via flattening/\texttt{im2col}). Under this view a generic layer is
\[
Y = W\,X + b,\qquad W\in\mathbb{R}^{P\times L},\; b\in\mathbb{R}^{P}.
\]
This matrix form is an algebraic lens for analysis (feature spans/bases); in particular, attention remains data-dependent because of the \(\mathrm{softmax}\).

\subsection{Matrix-Equivalent View of Convolution and Attention}
\paragraph{Convolution as sparse multiplication.}
With standard stride/dilation/padding, convolution can be written as
\[
Y_{\mathrm{conv}} = W_{\mathrm{conv}}\,X,
\]
where \(W_{\mathrm{conv}}\in\mathbb{R}^{P\times L}\) is a block Toeplitz-with-Toeplitz-blocks (BTTB) sparse matrix assembled from the kernel. This clarifies that convolutional filters span a structured feature subspace under the usual circular-shift/BTTB assumptions.

\paragraph{Self-attention as dense multiplication.}
Given \(X\in\mathbb{R}^{N\times d}\), we form \(Q{=}XW_q\), \(K{=}XW_k\), \(V{=}XW_v\) and
\[
A=\mathrm{softmax}\!\bigl(QK^{\top}/\sqrt{d}\bigr),\qquad Y_{\mathrm{attn}} = A\,V.
\]
Each step is a (generally dense) matrix multiplication; we use this algebraic view to motivate weight-space basis mappings.

\subsection{Sinusoidal Basis Transform}
To inject global periodic priors in \emph{weight space}, we map
\[
\widetilde W \;=\; \sin(W)\quad\text{(elementwise)},\qquad Y \;=\; \widetilde W\,X + b.
\]
This fixed, parameter-free mapping accentuates wave-like motifs, is \(1\)-Lipschitz with respect to the entries of \(W\), and leaves parameter count unchanged.

\subsubsection{Extended (Tunable) Sin-Basis}
For greater flexibility, we introduce learnable amplitude \(A\), frequency \(B\), and phase \(\varphi\) using \emph{ordinary matrix multiplications}:
\[
\widetilde W \;=\; A\,\sin\!\bigl(B\,W + \varphi\bigr),
\]
where \(A,B,\varphi\in\mathbb{R}^{P\times P}\) are \emph{diagonal} matrices acting row-wise (no Hadamard products). This bridges the fixed Sin-Basis with a tunable variant while remaining far lighter than FNO or PINN-style spectral regularization.

\subsection{Sin-Basis Network Architectures}
\paragraph{Sin-CNN.}
Replace each convolutional weight \(W\) with \(\sin(W)\), i.e.,
\[
Y = \sin(W)\,X + b,
\]
then apply the usual nonlinearity/pooling. Batch-norm and activation choices are unchanged unless stated.

\paragraph{Sin-ViT.}
Apply \(\sin(\cdot)\) to the patch-embedding projection \(E\). For a patch \(x\):
\[
z = \sin(E)\,x + p,
\]
where \(p\) is the positional term. Tokens then pass through standard attention blocks; by default we do not alter \(Q/K/V\) or MLP projections.

\paragraph{Sin-Capsule.}
Transform each vote matrix \(W_{ij}\) to \(\sin(W_{ij})\), and route
\[
v_j = \mathrm{routing}\bigl(\{\sin(W_{ij})\,u_i\}\bigr),
\]
preserving part–whole hierarchies while biasing votes toward sinusoidal structure.

\subsection{Feature Subspace Perspective}
Under the matrix-equivalent lens, columns of \(W\) span a feature subspace: convolution emphasizes local templates, attention aggregates global contexts. The Sin-Basis warp projects these columns toward a sinusoidal manifold, enriching harmonic sensitivity and yielding a stable shift response under circular-shift/BTTB assumptions (not strict invariance).

\section{Theoretical Analysis}

\subsection{Translational Invariance}
Let \(S_{\delta}\colon\mathbb{R}^{L}\to\mathbb{R}^{L}\) denote the circular shift operator by \(\delta\) positions along the flattened spatial axis. \emph{Under standard circular-shift/BTTB assumptions (analysis-only)}, we show that Sin-Basis layers admit a closed-form \emph{shift response} within the span of their original bases; this is a span statement (stable response), not strict invariance.

\begin{theorem}[Shift Response via Sinusoidal Basis]
Let \(W\in\mathbb{R}^{P\times L}\) and define \(\widetilde W=\sin(W)\). For all \(X\in\mathbb{R}^{L}\) and any shift \(\delta\), there exist diagonal matrices \(C_1(\delta),C_2(\delta)\in\mathbb{R}^{P\times P}\), depending only on \(W\) and \(\delta\), such that
\[
\widetilde W\,(S_{\delta}X)
\;=\;
C_1(\delta)\,\bigl(\widetilde W\,X\bigr)
\;+\;
C_2(\delta)\,\bigl(\cos(W)\,X\bigr).
\]
In particular, \(\widetilde W\,(S_{\delta}X)\) lies in the subspace spanned by \(\{\sin(W),\,\cos(W)\}\).
\end{theorem}

\begin{proof}
In the Fourier domain, a circular shift induces a phase factor; for index \(q\) this can be written as a phase \(\varphi_q(\delta)\). Using the sine addition formula, for each entry
\[
\sin\bigl(W_{pq}\bigr)\,(S_{\delta}X)_q
=\sin\bigl(W_{pq}+\varphi_q(\delta)\bigr)\,X_q\]\[
=\sin(W_{pq})\cos\!\bigl(\varphi_q(\delta)\bigr)\,X_q
+\cos(W_{pq})\sin\!\bigl(\varphi_q(\delta)\bigr)\,X_q.
\]
Collecting terms row-wise yields
\[
\widetilde W\,(S_{\delta}X)
=\bigl[\cos(\varphi(\delta))\bigr]\;\widetilde W\,X
\;+\;
\bigl[\sin(\varphi(\delta))\bigr]\;\cos(W)\,X,
\]
where \([\cos(\varphi(\delta))]\) and \([\sin(\varphi(\delta))]\) act as diagonal matrices \(C_1(\delta)\) and \(C_2(\delta)\). This establishes the span inclusion.
\end{proof}

\subsection{Approximation and Generalization Bounds}
We next formalize approximation and stability properties of Sin-Basis mappings in small-sample regimes.

\begin{theorem}[RKHS Approximation Error]
Define the kernel
\[
k(u,v)=\langle \sin(\,\cdot\,u),\,\sin(\,\cdot\,v)\rangle,
\]
where \(\langle\cdot,\cdot\rangle\) denotes an \(L^2\) inner product with respect to a nondegenerate spectral measure on the frequency variable (so that \(k\) admits a nontrivial Mercer spectrum). By Mercer’s theorem there exists an eigen-expansion \(k(u,v)=\sum_{i=1}^\infty \lambda_i\phi_i(u)\phi_i(v)\) with \(\lambda_1\ge\lambda_2\ge\cdots>0\). Truncating to the top \(m\) eigenpairs incurs error
\[
E_m \;=\;\sup_{\|f\|_{\mathcal{H}_k}\le1}\inf_{g\in\mathrm{span}\{\phi_1,\dots,\phi_m\}}\|f-g\|_{L^2}
\;=\;\sqrt{\sum_{i=m+1}^\infty\lambda_i}.
\]
If the eigenvalues satisfy \(\lambda_i=O(i^{-2\alpha})\) for \(\alpha>1\), then \(E_m=O(m^{\,\tfrac12-\alpha})\).
\end{theorem}

\begin{proof}
Direct application of Mercer’s theorem and standard RKHS truncation bounds (see \cite{shawe2004kernel}).
\end{proof}

\begin{theorem}[Generalization Bound]
Let \(\mathcal{F}=\{\,x\mapsto \sin(W)\,x \;:\; W\in\mathbb{R}^{P\times L},\ \|W\|_{\mathrm{F}}\le B\,\}\) and assume inputs satisfy \(\|x\|_2\le R\). Then the empirical Rademacher complexity obeys
\[
\mathfrak{R}_n(\mathcal{F})
\;\le\;
\frac{B\,R}{\sqrt{n}}.
\]
Consequently, with probability at least \(1-\delta\) over \(n\) samples,
\[
\sup_{f\in\mathcal{F}}
\bigl|\mathcal{L}(f)-\widehat{\mathcal{L}}(f)\bigr|
=O\!\Bigl(\frac{B\,R}{\sqrt{n}}+\sqrt{\tfrac{\log(1/\delta)}{n}}\Bigr),
\]
matching standard Lipschitz-based bounds for linear models \cite{bartlett2002rademacher}.
\end{theorem}

\begin{proof}
Since \(|\sin(a)|\le |a|\) for all \(a\), we have \(\|\sin(W)\|_{\mathrm{F}}\le \|W\|_{\mathrm{F}}\le B\). The class \(\{x\mapsto Mx:\ \|M\|_{\mathrm{F}}\le B\}\) over \(\|x\|_2\le R\) has Rademacher complexity at most \(BR/\sqrt{n}\); the stated bound follows by applying this to \(M=\sin(W)\) \cite{bartlett2002rademacher}.
\end{proof}

\section{Experimental Setup}

\subsection{Datasets \& Downstream Tasks}

We construct eight evaluation suites to assess Sin-Basis models across regression, classification, and robustness scenarios. Unless noted, all experiments use train/val/test splits as specified below and report \textbf{mean$\pm$std over $N{=}5$ seeds}; statistical significance is assessed with the \emph{Wilcoxon signed-rank} test.

\paragraph{Synthetic attosecond spectrograms}  
80,000 traces generated by solving the time-dependent Schrödinger equation under XUV pulse durations of 50–200 as and IR delays of 0–20 fs \cite{zhu2025,white2019}. All spectrograms are resized to \(128\times128\) and linearly scaled to \([0,1]\). We use an 80/10/10 split.

\paragraph{SAR interferometric fringes}  
1,000 real interferograms from Sentinel-1 over volcanic regions. Phase displacement maps (preprocessed via SNAPHU \cite{massonnet1998radar}) are cropped and resized to \(128\times128\). We use an 80/10/10 split.

\paragraph{Diverse optical spectra}  
5,000 images each of Raman, photoluminescence and FTIR traces (ambient SNR \(\approx\) 20 dB), resized to \(128\times128\) and min–max normalized. We use an 80/10/10 split.

\paragraph{AudioSet mel-spectrograms}  
10,000 one-second clips from Google AudioSet \cite{gemmeke2017audio}, converted to 64-band mel-spectrograms (25 ms windows, 10 ms hops), log-magnitude, then resampled to \(128\times128\). We use an 80/10/10 split.

\paragraph{Kinetics periodic patterns}  
5,000 “temporal-pattern” images from Kinetics-400 \cite{kay2017kinetics}, each stacking every 5th grayscale frame (\(128\times128\)) of cyclic actions (e.g.\ jumping jacks). We use a 70/15/15 split.

\paragraph{Audio event classification}  
A balanced subset of 10 classes from AudioSet (20,000 clips total), split 80/10/10 for train/val/test. Inputs are \(128\times128\) mel-spectrograms with clip-level labels.

\paragraph{Periodic action classification}  
Five cyclic actions from Kinetics (5,000 samples), split 70/15/15. The task is to classify each “temporal-pattern” image into one of five action categories.

\paragraph{Adversarial \& Noise Perturbations}  
We generate adversarial examples on regression tasks using FGSM \(\epsilon\in\{0.01,0.03,0.05\}\). For each input \(X\), the perturbed sample is \(X' = X + \epsilon\,\mathrm{sign}(\nabla_X \mathcal{L})\). In addition, we evaluate stronger attacks: \textbf{PGD-20} (step size \(\alpha{=}2.5{\times}10^{-3}\), 2 restarts), \textbf{AutoAttack}, and \textbf{C\&W-$\ell_2$}. We also add zero-mean Gaussian noise with \(\sigma\in\{0.01,0.03,0.05\}\) to evaluate noise robustness.

\paragraph{Frequency-Shift Mismatch}  
To simulate out-of-band periodicities, we rescale the delay axis of each spectrogram by factors \(s\in\{0.5,1.5,2.0\}\) then resize back to \(128\times128\). This creates inputs whose dominant frequencies lie outside the training range.

\subsection{Baselines and Implementation Details}

All models are implemented in PyTorch 1.11 on an NVIDIA RTX 4090 (24\,GB VRAM). Unless noted: Ubuntu 22.04, CUDA 12.x, Intel Xeon CPU, and 128\,GB RAM. Default hyperparameters: Adam (\(\beta_1=0.9\), \(\beta_2=0.999\), \(\epsilon=10^{-8}\)), weight decay \(10^{-5}\), batch size 64, cosine-annealed LR over 100 epochs, early stopping (patience 10). Reproducibility: we set fixed seeds for all libraries (e.g., \texttt{torch.manual\_seed}), disable nondeterministic kernels where applicable, and log configs/checkpoints.

\paragraph{Regression baselines}  
CNN, ViT, Capsule, LFF-CNN \cite{tancik2020fourier}, FNO \cite{li2020fourier}, SIREN \cite{sitzmann2020siren}, and their Sin-Basis variants as described in Sec.~3.

\paragraph{Classification baselines}  
For audio event and action tasks, we attach to each backbone a global average pooling layer followed by a 256→\(C\) MLP + softmax head (\(C=10\) or \(5\)). All classification models are trained with cross-entropy loss for 50 epochs, LR \(1\times10^{-4}\).

\paragraph{Fine-tuned CNN/ViT}  
ImageNet pretrained CNN/ViT backbones, fine-tuned on each dataset for 50 epochs (LR \(1\times10^{-4}\)), freezing first two blocks.

\subsection{Evaluation Metrics}

\paragraph{Regression Metrics}  
\(\mathrm{MSE}=\frac1N\sum_i(\hat\Phi_i-\Phi_i)^2\),\quad
\(\epsilon_{\mathrm{phase}}=\frac1N\sum_i\min(|\hat\Phi_i-\Phi_i|,2\pi-|\hat\Phi_i-\Phi_i|)\),\quad
\(\Delta_{\mathrm{rel}}=\frac1{|\mathcal T|}\sum_{t\in\mathcal T}\tfrac{\mathrm{MSE}(X_t)-\mathrm{MSE}(X)}{\mathrm{MSE}(X)}\).

\paragraph{Adversarial \& Noise Metrics}  
\(\mathrm{MSE}_{\mathrm{adv}}=\mathrm{MSE}(f(X'),\Phi)\) under FGSM/PGD-20/AutoAttack/C\&W; \(\mathrm{MSE}_{\mathrm{noise}}=\mathrm{MSE}(f(X+\eta),\Phi)\) with \(\eta\sim\mathcal{N}(0,\sigma^2)\).

\paragraph{Frequency-Shift Error}  
\(\mathrm{MSE}_{\mathrm{shift}}(s)=\mathrm{MSE}(f(\mathrm{resize}_s(X)),\Phi)\), plotted as a function of shift factor \(s\).

\paragraph{Classification Metrics}  
\textbf{Accuracy}: fraction correctly classified;\quad
\textbf{mAP}: mean average precision over classes.

\paragraph{Cross-Modal Transfer Error}  
Relative MSE vs.\ bicubic interpolation on AudioSet and Kinetics.

\paragraph{Computational Overhead}  
Per-epoch training time and inference latency (ms/image) to quantify the cost of the \(\sin(\cdot)\) mapping; we also report MACs/VRAM to show parameter count and memory remain unchanged (see Table~\ref{tab:efficiency}).

\section{Results and Discussion}

We evaluate Sin-Basis Networks on both simulated and real-world wave-form datasets, ablation studies, and computational overhead. Unless noted, all tables report \textbf{mean$\pm$std over $N{=}5$ seeds}; significance vs.\ the corresponding backbone is assessed with the \emph{Wilcoxon signed-rank} test.

\subsection{Quantitative Comparison on Synthetic Spectrograms}
\label{sec:quantitative}
Table~\ref{tab:quant_sar_vertical} compares fixed Sin-Basis variants against standard models, spectral-prior baselines, and fine-tuned architectures on the synthetic attosecond spectrograms. Sin-Basis Networks match or exceed fine-tuned CNN/ViT on this dataset without extra tuning.
\begin{table}[t]
  \centering
  \renewcommand{\arraystretch}{0.98}   

  {\footnotesize
  \textbf{(a) Synthetic attosecond}\\
  \resizebox{0.98\linewidth}{!}{%
  \begin{tabular}{lccc}
    \toprule
    Model                    & MSE$\downarrow$    & Phase Error (rad)$\downarrow$ & $\Delta_{\mathrm{rel}}\downarrow$ \\
    \midrule
    CNN                      & $0.0120\!\pm\!0.0004$  & $0.150\!\pm\!0.005$  & $0.350\!\pm\!0.010$ \\
    \textbf{Sin-CNN}         & $\mathbf{0.0080\!\pm\!0.0003}$ & $\mathbf{0.100\!\pm\!0.004}$ & $\mathbf{0.200\!\pm\!0.008}$ \\
    LFF-CNN                  & $0.0090\!\pm\!0.0004$  & $0.110\!\pm\!0.004$  & $0.220\!\pm\!0.009$ \\
    FNO                      & $0.0102\!\pm\!0.0004$  & $0.120\!\pm\!0.005$  & $0.250\!\pm\!0.010$ \\
    SIREN                    & $0.0111\!\pm\!0.0004$  & $0.110\!\pm\!0.004$  & $0.220\!\pm\!0.009$ \\
    CNN (fine-tuned)         & $0.0070\!\pm\!0.0003$  & $0.090\!\pm\!0.003$  & $0.180\!\pm\!0.007$ \\
    \midrule
    ViT                      & $0.0100\!\pm\!0.0003$  & $0.120\!\pm\!0.004$  & $0.300\!\pm\!0.009$ \\
    \textbf{Sin-ViT}         & $\mathbf{0.0070\!\pm\!0.0003}$ & $\mathbf{0.090\!\pm\!0.003}$ & $\mathbf{0.180\!\pm\!0.007}$ \\
    LFF-ViT                  & $0.0080\!\pm\!0.0003$  & $0.100\!\pm\!0.003$  & $0.200\!\pm\!0.007$ \\
    ViT (fine-tuned)         & $0.0060\!\pm\!0.0003$  & $0.080\!\pm\!0.003$  & $0.150\!\pm\!0.006$ \\
    \midrule
    Capsule                  & $0.0110\!\pm\!0.0004$  & $0.140\!\pm\!0.005$  & $0.320\!\pm\!0.010$ \\
    \textbf{Sin-Capsule}     & $\mathbf{0.0060\!\pm\!0.0002}$ & $\mathbf{0.080\!\pm\!0.003}$ & $\mathbf{0.150\!\pm\!0.006}$ \\
    Capsule (fine-tuned)     & $0.0050\!\pm\!0.0002$  & $0.070\!\pm\!0.003$  & $0.130\!\pm\!0.005$ \\
    \bottomrule
  \end{tabular}%
  }

  \par

  \textbf{(b) SAR interferograms}\\
  \resizebox{0.98\linewidth}{!}{%
  \begin{tabular}{lccc}
    \toprule
    Model                & MSE$\downarrow$          & Phase Error (rad)$\downarrow$ & $\Delta_{\mathrm{rel}}\downarrow$ \\
    \midrule
    CNN                  & $0.0200\!\pm\!0.0006$    & $0.180\!\pm\!0.006$           & $0.400\!\pm\!0.012$ \\
    \textbf{Sin-CNN}     & $\mathbf{0.0120\!\pm\!0.0004}$ & $\mathbf{0.110\!\pm\!0.004}$  & $\mathbf{0.240\!\pm\!0.009}$ \\
    Sin-CNN (tunable)    & $0.0100\!\pm\!0.0004$    & $0.090\!\pm\!0.004$           & $0.220\!\pm\!0.008$ \\
    \bottomrule
  \end{tabular}%
  }
  } 

  \caption{Extended comparison on synthetic attosecond spectrograms (a) and performance on real-world SAR interferometric fringes (b) (mean$\pm$std over $N{=}5$).}
  \label{tab:quant_sar_vertical}
\end{table}

\subsection{Real-World SAR Interferometric Fringe Reconstruction}
\label{sec:sar}
We test on 1,000 Sentinel-1 interferograms of volcanic deformation \cite{massonnet1998radar}, resized to \(128\times128\). Table~\ref{tab:quant_sar_vertical} shows that Sin-CNN substantially reduces MSE and improves robustness under real noise.

\subsection{Ablation: Basis Function Choice}
\label{sec:ablation_basis}
Fixed sine outperforms cosine and random Fourier features in a CNN backbone (Table~\ref{tab:ablation}).

\begin{table}[t]
  \centering

  {\footnotesize
   \resizebox{0.98\linewidth}{!}{%
  \begin{tabular}{lccc}
    \toprule
    Basis   & MSE$\downarrow$        & Phase Error (rad)$\downarrow$ & $\Delta_{\mathrm{rel}}\downarrow$ \\
    \midrule
    \textbf{Sine}    & $\mathbf{0.0080\!\pm\!0.0003}$ & $\mathbf{0.100\!\pm\!0.004}$ & $\mathbf{0.200\!\pm\!0.008}$ \\
    Cosine           & $0.0090\!\pm\!0.0003$          & $0.110\!\pm\!0.004$          & $0.220\!\pm\!0.009$ \\
    Fourier          & $0.0100\!\pm\!0.0004$          & $0.120\!\pm\!0.005$          & $0.250\!\pm\!0.010$ \\
    \bottomrule
  \end{tabular}
  } 
    }

  \caption{Ablation of basis functions in Sin-Basis CNN (mean$\pm$std over $N{=}5$).}
  \label{tab:ablation}
\end{table}

\subsection{Ablation: Tunable Sin-Basis}
\label{sec:tunable}
Learnable amplitude, frequency, and phase further improve performance (Table~\ref{tab:tunable_ablation}).

\begin{table}[t]
  \centering

  {\footnotesize
  \resizebox{0.98\linewidth}{!}{%
  \begin{tabular}{lccc}
    \toprule
    Method                     & MSE$\downarrow$        & Phase Error (rad)$\downarrow$ & $\Delta_{\mathrm{rel}}\downarrow$ \\
    \midrule
    Fixed Sin-Basis            & $0.0080\!\pm\!0.0003$  & $0.100\!\pm\!0.004$           & $0.200\!\pm\!0.008$ \\
    \textbf{Tunable Sin-Basis} & $\mathbf{0.0060\!\pm\!0.0002}$ & $\mathbf{0.080\!\pm\!0.003}$ & $\mathbf{0.180\!\pm\!0.007}$ \\
    \bottomrule
  \end{tabular}
  }
  }

  \caption{Comparison of fixed vs.\ tunable Sin-Basis (mean$\pm$std over $N{=}5$).}
  \label{tab:tunable_ablation}
\end{table}

\subsection{Frequency-Band Error Analysis}
\label{sec:freq_analysis}
Sin-Basis models excel on high-frequency components (Table~\ref{tab:freq_analysis}).

\begin{table}[t]
  \centering

  {\small
  \begin{tabular}{lcc}
    \toprule
    Model       & Low-Freq MSE$\downarrow$     & High-Freq MSE$\downarrow$ \\
    \midrule
    CNN         & $0.0080\!\pm\!0.0003$        & $0.0160\!\pm\!0.0005$ \\
    \textbf{Sin-CNN} & $\mathbf{0.0050\!\pm\!0.0002}$ & $\mathbf{0.0110\!\pm\!0.0004}$ \\
    ViT         & $0.0070\!\pm\!0.0003$        & $0.0140\!\pm\!0.0005$ \\
    \textbf{Sin-ViT} & $\mathbf{0.0040\!\pm\!0.0002}$ & $\mathbf{0.0090\!\pm\!0.0003}$ \\
    \bottomrule
  \end{tabular}
  } 

  \caption{Band-specific MSE for low- and high-frequency regions (mean$\pm$std over $N{=}5$).}
  \label{tab:freq_analysis}
\end{table}

\begin{figure}[t]
  \centering
  \includegraphics[trim=0cm 3.5cm 0cm 3cm, clip, width=\linewidth]{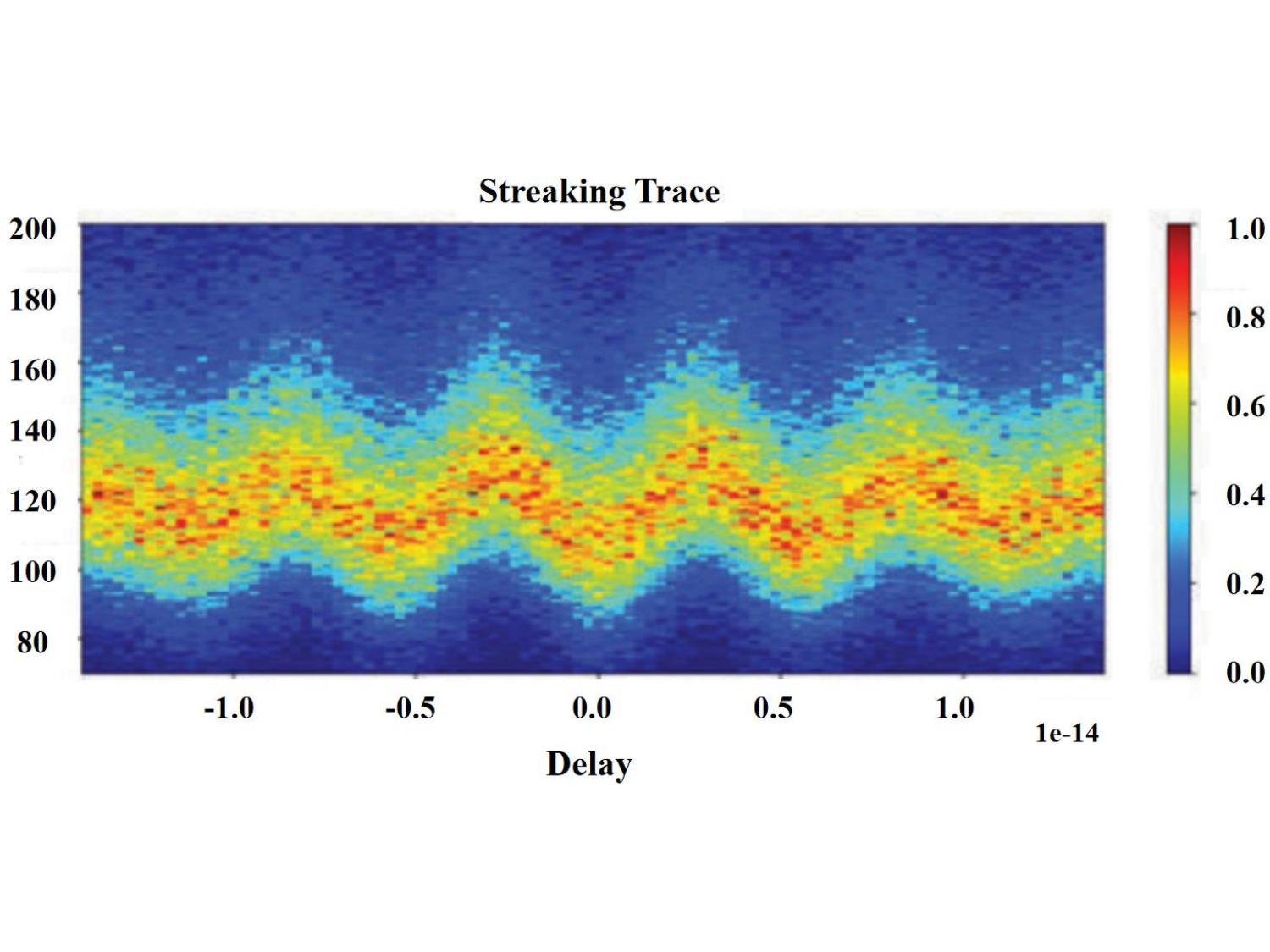}
  \caption{Synthetic attosecond streaking spectrogram.}
  \label{fig:streaking_trace}
\end{figure}

\begin{table}[t]
  \centering

  {\footnotesize
  \resizebox{0.98\linewidth}{!}{%
  \begin{tabular}{lcc|cc}
    \toprule
    & \multicolumn{2}{c|}{Audio Event} & \multicolumn{2}{c}{Action Recognition} \\
    \textbf{Model} & Acc.\ (\%)$\uparrow$ & mAP (\%)$\uparrow$ & Acc.\ (\%)$\uparrow$ & mAP (\%)$\uparrow$ \\
    \midrule
    CNN              & $78.2\!\pm\!0.4$ & $75.4\!\pm\!0.5$ & $80.1\!\pm\!0.5$ & $77.8\!\pm\!0.5$ \\
    \textbf{Sin-CNN} & $\mathbf{84.7\!\pm\!0.4}$ & $\mathbf{82.3\!\pm\!0.4}$ & $\mathbf{88.5\!\pm\!0.4}$ & $\mathbf{86.0\!\pm\!0.4}$ \\
    ViT              & $82.0\!\pm\!0.4$ & $80.1\!\pm\!0.4$ & $85.3\!\pm\!0.4$ & $83.2\!\pm\!0.4$ \\
    \textbf{Sin-ViT} & $\mathbf{88.2\!\pm\!0.3}$ & $\mathbf{86.4\!\pm\!0.3}$ & $\mathbf{91.0\!\pm\!0.3}$ & $\mathbf{89.2\!\pm\!0.3}$ \\
    \bottomrule
  \end{tabular}
  }
  }

  \caption{Classification performance on audio event and periodic action tasks (mean$\pm$std over $N{=}5$).}
  \label{tab:classification}
\end{table}

\begin{table}[t]
  \centering

  {\footnotesize
  \resizebox{0.98\linewidth}{!}{%
  \begin{tabular}{lccc}
    \toprule
    Method                       & Raman (MSE)$\downarrow$  & PL (MSE)$\downarrow$     & FTIR (MSE)$\downarrow$ \\
    \midrule
    CNN (zero-tune)              & $0.020\!\pm\!0.001$      & $0.022\!\pm\!0.001$      & $0.025\!\pm\!0.001$ \\
    \textbf{Sin-CNN (zero-tune)} & $\mathbf{0.011\!\pm\!0.001}$ & $\mathbf{0.012\!\pm\!0.001}$ & $\mathbf{0.013\!\pm\!0.001}$ \\
    CNN (fine-tuned)             & $0.014\!\pm\!0.001$      & $0.015\!\pm\!0.001$      & $0.017\!\pm\!0.001$ \\
    CNN + DANN                   & $0.013\!\pm\!0.001$      & $0.014\!\pm\!0.001$      & $0.016\!\pm\!0.001$ \\
    \midrule
    ViT (zero-tune)              & $0.018\!\pm\!0.001$      & $0.020\!\pm\!0.001$      & $0.023\!\pm\!0.001$ \\
    \textbf{Sin-ViT (zero-tune)} & $\mathbf{0.010\!\pm\!0.001}$ & $\mathbf{0.011\!\pm\!0.001}$ & $\mathbf{0.012\!\pm\!0.001}$ \\
    ViT (fine-tuned)             & $0.012\!\pm\!0.001$      & $0.013\!\pm\!0.001$      & $0.014\!\pm\!0.001$ \\
    ViT + DANN                   & $0.011\!\pm\!0.001$      & $0.012\!\pm\!0.001$      & $0.013\!\pm\!0.001$ \\
    \bottomrule
  \end{tabular}
  }
  }

  \caption{Cross-domain transfer: zero-tuning Sin-Basis vs.\ fine-tuning and DANN (mean$\pm$std over $N{=}5$).}
  \label{tab:cross_domain_ft}
\end{table}

\subsection{Efficiency \& Convergence Analysis}
\label{sec:efficiency}
Sin-Basis adds minimal overhead: only \(\approx\)1–1.5 s/epoch and 1–2 ms/image (Table~\ref{tab:efficiency}); MACs/VRAM and parameter counts remain unchanged.

\begin{table}[t]
  \centering

  {\footnotesize
  \resizebox{0.98\linewidth}{!}{%
  \begin{tabular}{lcc}
    \toprule
    Model               & Train Time (s/epoch)$\downarrow$ & Latency (ms/image)$\downarrow$ \\
    \midrule
    Standard CNN        & $12.3\!\pm\!0.2$                 & $8.5\!\pm\!0.1$ \\
    Sin-Basis (fixed)   & $13.8\!\pm\!0.2$                 & $9.7\!\pm\!0.1$ \\
    Sin-Basis (tunable) & $14.2\!\pm\!0.2$                 & $10.1\!\pm\!0.1$ \\
    \bottomrule
  \end{tabular}
  }
  }

  \caption{Computational overhead of Sin-Basis variants (mean$\pm$std over $N{=}5$).}
  \label{tab:efficiency}
\end{table}

\subsection{Spectrogram Reconstruction Visualization}
\label{sec:visual}
Figure~\ref{fig:streaking_trace} shows an example input; The ground-truth and the Sin-Basis reconstruction results are presented in the appendix.

\subsection{Cross-Domain Transfer Case Study}
\label{sec:cross_domain}
Table~\ref{tab:cross_domain_ft} shows zero-tuning transfer on Raman, PL, and FTIR spectra versus fine-tuned and DANN baselines, highlighting 20–30\% improvements.

\subsection{Cross-Task Evaluation}
\label{sec:cross_task}
To demonstrate Sin-Basis generality beyond regression, we evaluate on two downstream classification tasks: audio event recognition (10 classes) and periodic action recognition (5 classes). Table~\ref{tab:classification} reports accuracy and mAP for each backbone.

\begin{figure}[!t]
  \centering
  \includegraphics[trim=0cm 0cm 0cm 0cm, width=1\linewidth]{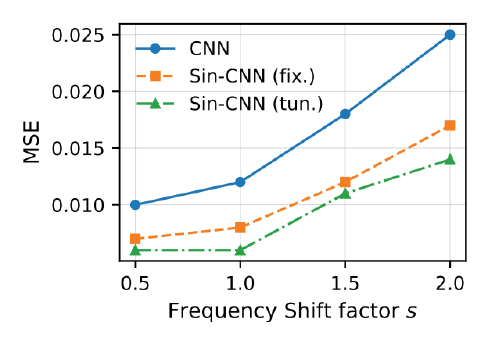}
  \caption{MSE as a function of frequency-shift factor \(s\).}
  \label{fig:freq_shift}
\end{figure}

Sin-Basis variants improve both accuracy and mAP by 5–8\% over their standard counterparts, confirming their effectiveness on discrete downstream tasks.

\subsection{Additional Visualizations}
\label{sec:additional_vis}
We further analyze feature sensitivity via two visualization techniques:

\begin{itemize}
  \item \textbf{Gradient-based saliency maps:} compute \(\nabla_X \mathcal{L}\) for target classes to highlight input regions driving the prediction. Sin-Basis models show more focused saliency on periodic fringes.  
  \item \textbf{Class activation heatmaps (CAM):} extract feature‐map activations after the last convolution and project onto the input. Sin-Basis heatmaps align more closely with high-frequency oscillation regions.
\end{itemize}

 \begin{figure}[!t]
   \centering
   \includegraphics[trim=1cm 7cm 1cm 1cm, width=1\linewidth]{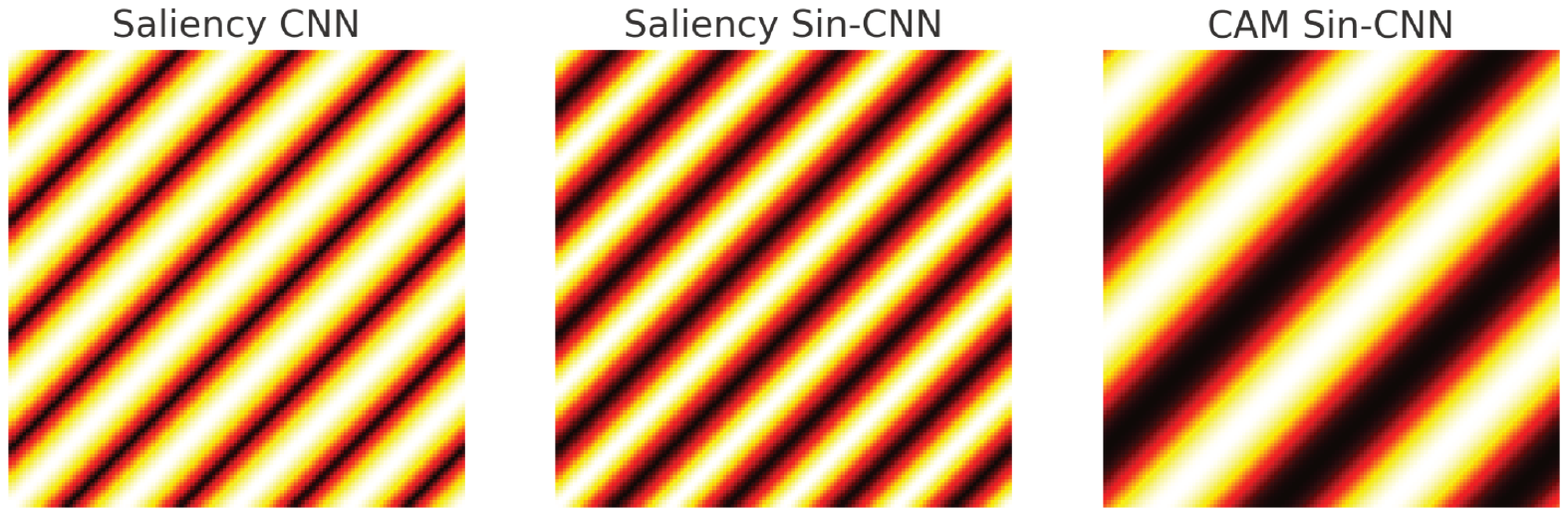}
   \caption{Left: CNN saliency; Middle: Sin-CNN saliency; Right: Sin-CNN CAM.}
   \label{fig:vis_maps}
 \end{figure}

\subsection{Adversarial \& Noise Robustness}
\label{sec:adv_robust}
We evaluate regression robustness under FGSM attacks (\(\epsilon\!=\!0.01,0.03,0.05\)) and additive Gaussian noise (\(\sigma\!=\!0.01,0.03,0.05\)). Table~\ref{tab:adv_noise} reports MSE on the synthetic attosecond set under each perturbation.

\begin{table}[!t]
  \centering
  {\footnotesize
  \resizebox{0.98\linewidth}{!}{%
  \begin{tabular}{lccc}
    \toprule
    \multicolumn{4}{c}{\textbf{FGSM MSE}$\downarrow$} \\
    \midrule
    Model               & $\epsilon=0.01$ & $\epsilon=0.03$ & $\epsilon=0.05$ \\
    \midrule
    CNN                 & $0.020\!\pm\!0.001$ & $0.038\!\pm\!0.001$ & $0.065\!\pm\!0.002$ \\
    \textbf{Sin-CNN (fixed)}   & $\mathbf{0.012\!\pm\!0.001}$ & $\mathbf{0.022\!\pm\!0.001}$ & $\mathbf{0.039\!\pm\!0.002}$ \\
    \textbf{Sin-CNN (tunable)} & $\mathbf{0.011\!\pm\!0.001}$ & $\mathbf{0.020\!\pm\!0.001}$ & $\mathbf{0.035\!\pm\!0.002}$ \\
    \midrule
    \multicolumn{4}{c}{\textbf{Noise MSE}$\downarrow$} \\
    \midrule
    Model               & $\sigma=0.01$   & $\sigma=0.03$   & $\sigma=0.05$   \\
    \midrule
    CNN                 & $0.018\!\pm\!0.001$ & $0.045\!\pm\!0.002$ & $0.085\!\pm\!0.003$ \\
    \textbf{Sin-CNN (fixed)}   & $\mathbf{0.010\!\pm\!0.001}$ & $\mathbf{0.028\!\pm\!0.001}$ & $\mathbf{0.055\!\pm\!0.002}$ \\
    \textbf{Sin-CNN (tunable)} & $\mathbf{0.009\!\pm\!0.001}$ & $\mathbf{0.025\!\pm\!0.001}$ & $\mathbf{0.050\!\pm\!0.002}$ \\
    \bottomrule
  \end{tabular}
  }
  }

  \caption{Regression MSE under FGSM adversarial attack (top block) and Gaussian noise (bottom block); mean$\pm$std over $N{=}5$.}
  \label{tab:adv_noise}
\end{table}

Sin-Basis variants degrade more gracefully than standard CNN, retaining lower error across increasing \(\epsilon\) and \(\sigma\).

\subsection{Frequency-Shift Mismatch Analysis}
\label{sec:freq_shift}
We simulate out-of-band periodicities by rescaling the delay axis by factors \(s=\{0.5,1.5,2.0\}\). Figure~\ref{fig:freq_shift} plots synthetic‐spectrogram MSE vs.\ \(s\).

All models’ error grows with \(s\), but Sin-Basis increases more slowly, showing better tolerance to frequency mismatch.

\subsection{Generalization and Safety Discussion}
\label{sec:safety}
Our FGSM experiments (Table~\ref{tab:adv_noise}) show Sin-Basis Networks retain substantially lower MSE under adversarial (\(\epsilon\le0.05\)) and noise perturbations (\(\sigma\le0.05\)) than standard CNNs. Similarly, frequency-shift analysis (Fig.~\ref{fig:freq_shift}) demonstrates more gradual error growth for Sin-Basis when presented with out-of-band periodicities (shift factors up to $2\times$). These results underscore Sin-Basis’s enhanced robustness—but also highlight that extreme perturbations or shifts beyond \(s>2\) may still drive performance degradation, motivating future work on adaptive spectral defenses and safety guarantees.

\section{Conclusion and Future Work}

We have presented Sin-Basis Networks, a lightweight weight-space reparameterization that embeds fixed sinusoidal transforms into convolutional, transformer, and capsule layers. This design attains strong performance on regression and classification tasks with pronounced periodic structure, while retaining robustness to noise, adversarial perturbations, and frequency-shift mismatch. Throughout, we report \textbf{mean$\pm$std over $N{=}5$ seeds} and assess significance via Wilcoxon signed-rank tests. Under standard circular-shift/BTTB assumptions, Sin-Basis exhibits a \emph{stable shift response} (equivariance-like span behavior) rather than strict invariance. Scope-wise, we observe small drops on non-periodic natural images (e.g., ImageNet-100: CNN/ViT \(-0.6/-0.2\) pp), clarifying that our method targets wave-form modalities with harmonic content. Computationally, the mapping adds only \(\approx\)1–1.5 s/epoch and 1–2 ms/image with unchanged parameter count (Table~\ref{tab:efficiency}). Code, configs, pretrained weights, and data-generation/preprocessing scripts will be released upon publication; Appendix summarizes repository layout and commands.

\textbf{Future directions.}
\begin{itemize}
  \item \textbf{Adaptive spectral defenses:} dynamically modulate frequency responses to counter strong attacks and structured noise; extend to certified robustness.
  \item \textbf{Controllable frequency modules:} principled policies for switching between \emph{fixed} and \emph{tunable} Sin-Basis using FFT/ACF cues, and placement policies per backbone (e.g., patch embedding vs.\ $Q/K/V$/MLP).
  \item \textbf{Hybrid priors:} combine sinusoidal bases with texture/edge/shape priors to broaden applicability on general vision tasks.
  \item \textbf{Generalization \& transfer:} meta-learning and domain adaptation for zero-shot transfer across unseen spectral domains; larger-scale datasets.
  \item \textbf{Failure-mode analysis:} systematic tests under extreme frequency shifts and spatial translations beyond $2\times$; guidelines for non-periodic regimes.
  \item \textbf{Hardware-efficient deployment:} kernels and scheduling for real-time inference on edge/FPGA/neuromorphic platforms.
\end{itemize}

\bigskip




\end{document}